\documentclass{llncs}
\usepackage{wrapfig}

\usepackage{graphicx}
\usepackage{epsfig}
\usepackage{verbatim}

\usepackage{algorithm,algorithmic}
\usepackage{setspace, floatflt}
\usepackage[hyphens]{url}

\begin{document}

\newcommand{\shr}[1]{\resizebox{!}{0.20cm}{\mbox{$#1$}}}
\newcommand{\sh}[1]{\resizebox{!}{0.25cm}{\mbox{$#1$}}}

\title{The Stochastic Gradient Descent for the Primal L1-SVM Optimization Revisited}
\author{Constantinos Panagiotakopoulos \and Petroula Tsampouka}
\institute{School of Technology, Aristotle University of Thessaloniki, Greece \\
\email{costapan@eng.auth.gr, petroula@auth.gr} }
\maketitle

\begin{abstract}
We reconsider the stochastic (sub)gradient approach to the unconstrained primal L1-SVM optimization. We observe that if the learning rate is inversely proportional to the number of steps, i.e., the number of times any training pattern is presented to the algorithm, the update rule may be transformed into the one of the classical perceptron with margin in which the margin threshold increases linearly with the number of steps. Moreover, if we cycle repeatedly through the possibly randomly permuted training set the dual variables defined naturally via the expansion of the weight vector as a linear combination of the patterns on which margin errors were made are shown to obey at the end of each complete cycle automatically the box constraints arising in dual optimization. This renders the dual Lagrangian a running lower bound on the primal objective tending to it at the optimum and makes available an upper bound on the relative accuracy achieved which provides a meaningful stopping criterion. In addition, we propose a mechanism of presenting the same pattern repeatedly to the algorithm which maintains the above properties. Finally, we give experimental evidence that algorithms constructed along these lines exhibit a considerably improved performance.
\end{abstract}
\renewcommand{\vec}[1]{\mbox{\boldmath$#1$}}
\newcommand{\Tiny}[1]{\mbox{\tiny$#1$}}
\section{Introduction}
Support Vector Machines (SVMs) \cite{Bos,Vap,CST} have been extensively used as linear classifiers either in the space where the patterns originally reside or in high dimensional feature spaces induced by kernels. They appear to be very successful at addressing the classification problem expressed as the minimization of an objective function involving the empirical risk while at the same time keeping low the complexity of the classifier. As measures of the empirical risk various quantities have been proposed with the 1- and 2-norm loss functions being the most widely accepted ones giving rise to the optimization problems known as L1- and L2-SVMs \cite{Cortes}. SVMs typically treat the problem as a constrained quadratic optimization in the dual space. At the early stages of SVM development their efficient implementation was hindered by the quadratic dependence of their memory requirements on the number of training examples a fact which rendered prohibitive the processing of large datasets. The idea of applying optimization only to a subset of the training set in order to overcome this difficulty resulted in the development of decomposition methods \cite{Plat,Joa}. Although such methods led to improved convergence rates, in practice their superlinear dependence on the number of examples, which can be even cubic, can still lead to excessive runtimes when dealing with massive datasets. Recently, the so-called linear SVMs \cite{Joa06,FS,HCL,PT2} taking advantage of linear kernels in order to allow parts of them to be written in primal notation succeeded in outperforming decomposition SVMs.

The above considerations motivated research in alternative algorithms naturally formulated in primal space long before the advent of linear SVMs mostly in connection with the large margin classification of linearly separable datasets a problem directly related to the L2-SVM. Indeed, in the case that the 2-norm loss takes the place of the empirical risk an equivalent formulation exists which renders the dataset linearly separable in a high dimensional feature space. Such alternative algorithms (\cite{PT1,PT3} and references therein) are mostly based on the perceptron \cite{Ros}, the simplest online learning algorithm for binary linear classification, with their key characteristic being that they work in the primal space in an online manner, i.e., processing one example at a time. Cycling repeatedly through the patterns they update their internal state stored in the weight vector each time an appropriate condition is satisfied. This way, due to their ability to process one example at a time, such algorithms succeed in sparing time and memory resources and consequently become able to handle large datasets.

Since the L1-SVM problem is not known to admit an equivalent maximum margin interpretation via a mapping to an appropriate space fully primal large margin perceptron-like algorithms appear unable to deal with such a task.\footnote{The Margin Perceptron with Unlearning (MPU) \cite{PT2} addresses the L1-SVM problem by keeping track of the number of updates caused by each pattern in parallel with the weight vector which is updated according to a perceptron-like rule. In that sense MPU uses dual variables and should rather be considered a linear SVM which, however, possesses a finite time bound for achieving a predefined relative accuracy.} Nevertheless, a somewhat different approach giving rise to online algorithms was developed  which focuses on the minimization of the regularized 1-norm soft margin loss through stochastic gradient descent (SGD). Notable representatives of this approach are the pioneer NORMA \cite{KSW} (see also \cite{Zh}) and Pegasos \cite{SSS,SSS1} (see also \cite{Bor,Bot}). SGD gives rise to a kind of perceptron-like update having as an important ingredient the ``shrinking" of the current weight vector. Shrinking always takes place when a pattern is presented to the algorithm with it being the only modification suffered by the weight vector if no loss is incurred. Thus, due to lack of a meaningful stopping criterion the algorithm without user intervention keeps running forever. In that sense the algorithms in question are fundamentally different from the mistake-driven large margin perceptron-like classifiers which terminate after a finite number of updates. There is no proof even for their asymptotic convergence when they use as output the final hypothesis but they do exist probabilistic convergence results or results in terms of the average hypothesis. 

In the present work we reconsider the straightforward version of SGD for the primal unconstrained L1-SVM problem assuming a learning rate inversely proportional to the number of steps. Therefore, such an algorithm can be regarded either as NORMA with a specific dependence of the learning rate on the number of steps or as Pegasos with no projection step in the update and with a single example contributing to the (sub)gradient ($k=1$). We observe here that this algorithm may be transformed into a classical perceptron with margin \cite{DH} in which the margin threshold increases linearly with the number of steps. The obvious gain from this observation is that the shrinking of the weight vector at each step amounts to nothing but an increase of the step counter by one unit instead of the costly multiplication of all the components of the generally non-sparse weight vector with a scalar. Another benefit arising from the above simplified description is that we are able to demonstrate easily that if we cycle through the data in complete epochs the dual variables defined naturally via the expansion of the weight vector as a linear combination of the patterns on which margin errors were made satisfy automatically the box constraints of the dual optimization. An important consequence of this unexpected result is that the relevant dual Lagrangian which is expressed in terms of the total number of margin errors, the number of complete epochs and the length of the current weight vector provides during the run a lower bound on the primal objective function and gives us a measure of the progress made in the optimization process. Indeed, by virtue of the strong duality theorem the dual Lagrangian and the primal objective coincide at optimality. Therefore, assuming convergence to the optimum an upper bound on the relative accuracy involving the dual Lagrangian may be defined which offers a useful and practically achievable stopping criterion. Moreover, we may now provide evidence in favor of the asymptotic convergence to the optimum by testing experimentally the vanishing of the duality gap. Finally, aiming at performing more updates at the expense of only one costly inner product calculation we propose a mechanism of presenting the same pattern repeatedly to the algorithm consistently with the above interesting properties.

The paper is organized as follows. Section 2 describes the algorithm and its properties. In Section 3 we give implementational details and deliver our experimental results. Finally, Section 4 contains our conclusions.  
  
\section{The Algorithm and its Properties} 
Assume we are given a training set $\{(\vec x_k, l_k)\}^m_{k=1}$, with vectors $\vec x_k\in \bbbr^d$ and labels $l_k \in \{+1,-1\}$. This set may be either the original dataset or the result of a mapping into a feature space of higher dimensionality \cite{Vap,CST}. By placing $\vec x_k$ in the same position at a distance $\rho$ in an additional dimension, i.e., by extending $\vec x_k$ to $[\vec x_k, \rho]$, we construct an embedding of our data into the so-called augmented space \cite{DH}. The advantage of this embedding is that the linear hypothesis in the augmented space becomes homogeneous. Following the augmentation, a reflection with respect to the origin of the negatively labeled patterns is performed allowing for a uniform treatment of both categories of patterns. We define $R\equiv\displaystyle \max_{k} \left\| \vec{y}_{k} \right\|$ with $\vec{y}_{k} \equiv [l_k\vec x_k, l_k\rho]$ the $k$-th augmented and reflected pattern.

Let us consider the regularized empirical risk 
\[
\frac{\lambda}{2}\left\|\vec w\right\|^2+\frac{1}{m} \sum_{k=1}^m \max \{0,1-\vec w \cdot \vec y_k\}
\]
involving the 1-norm soft margin loss $\max \{0,1-\vec w \cdot \vec y_k\}$ for the pattern $\vec{y}_{k}$ and the regularization parameter $\lambda>0$ controlling the complexity of the classifier $\vec{w}$. For a given dataset of size $m$ minimization of the regularized empirical risk with respect to $\vec w$ is equivalent to the minimization of the objective function
\[
{\cal J}(\vec w, C)\equiv \frac{1}{2}\left\|\vec w\right\|^2+C \sum_{k=1}^m \max \{0,1-\vec w \cdot \vec y_k\} \enspace,
\]
where the ``penalty'' parameter $C>0$ is related to $\lambda$ as
\[
C=\frac{1}{\lambda m}\enspace.
\]
This is the L1-SVM problem expressed as an unconstrained optimization.

The algorithms we are concerned with are classical SGD algorithms. The term stochastic refers to the fact that they perform gradient descent with respect to the objective function in which the empirical risk $({1}/{m}) \sum_{k=1}^m \max \{0,1-\vec w \cdot \vec y_k\}$ is approximated by the instantaneous risk $ \max \{0,1-\vec w \cdot \vec y_k\}$ on a single example. The general form of the update rule is then
\[
\vec w_{t+1}= \vec w_{t}-\eta_t \nabla_{\vec w_t}\left[\frac{1}{2}\left\|\vec w_t\right\|^2+\frac{1}{\lambda}\max \{0,1-\vec w_t \cdot \vec y_k\}\right] \enspace,
\]
where $\eta_t$ is the learning rate and $\nabla_{\vec w_t}$ stands for a subgradient with respect to $\vec w_{t}$ since the 1-norm soft margin loss is only piecewise differentiable ($t \ge 0$). We choose a learning rate $\eta_t=1/(t+1)$ which satisfies the conditions $\sum^{\infty}_{t=0}\eta^2_t<\infty$ and $\sum^{\infty}_{t=0}\eta_t=\infty$ usually imposed in the convergence analysis of stochastic approximations. Then, noticing that $\vec w_{t}-\frac{1}{t+1}\vec w_{t}=\frac{t}{t+1}\vec w_{t}$, we obtain the update
\begin{equation}
\label{ubd1}
\vec w_{t+1}= \frac{t}{t+1}\vec w_{t}+\frac{1}{\lambda(t+1)}\vec y_k
\end{equation}
whenever
\begin{equation}
\label{cond1}
\vec w_t \cdot \vec y_k \le 1
\end{equation}
and
\begin{equation}
\label{upd2}
\vec w_{t+1}= \frac{t}{t+1}\vec w_t
\end{equation}
otherwise. In deriving the above update rule we made the choice $\vec w_t-\lambda^{-1}\vec y_k$ for the subgradient at the point $\vec w_t \cdot \vec y_k=1$ where the 1-norm soft margin loss is not differentiable. We assume that $\vec w_0=\vec0$. We see that if $\vec w_t \cdot \vec y_k > 1$ the update consists of a pure shrinking of the current weight vector by the factor $t/(t+1)$.

The update rule may be simplified considerably if we perform the change of variable
\begin{equation}
\label{var}
\vec w_t=\frac{\vec a_t}{\lambda t}
\end{equation}
for $t>0$ and $\vec w_0 =\vec a_0=\vec0$ for $t=0$. In terms of the new weight vector $\vec a_t$ the update rule becomes
\begin{equation}
\label{upd3}
\vec a_{t+1}= \vec a_t+\vec y_k
\end{equation}
whenever
\begin{equation}
\label{cond2}
\vec a_t \cdot \vec y_k \le \lambda t
\end{equation}
and
\begin{equation}
\label{upd4}
\vec a_{t+1}= \vec a_t
\end{equation}
otherwise.\footnote{For $t=0$ (\ref{cond2}) becomes $\vec a_0 \cdot \vec y_k \le 0$ instead of $\vec a_0 \cdot \vec y_k \le 1$ which is obtained from (\ref{cond1}) with $\vec w_0 =\vec a_0$. Since both are satisfied with $\vec a_0=\vec 0$ (\ref{cond2}) may be used for all $t$.}  This is the update of the classical perceptron algorithm with margin in which, however, the margin threshold in condition (\ref{cond2}) increases linearly with the number of presentations of patterns to the algorithm independent of whether they lead to a change in the weight vector $\vec a_{t}$. Thus, $t$ counts the number of times any pattern is presented to the algorithm which corresponds to the number of updates (including the pure shrinkings (\ref{upd2})) of the weight vector $\vec w_{t}$. Instead, the weight vector $\vec a_{t}$ is updated only if (\ref{cond2}) is satisfied meaning that a margin error is made on $\vec y_k$.

In the original formulation of Pegasos \cite{SSS} the update is completed with a projection step in order to enforce the bound $\left\|\vec{w}_t\right\| \le 1/\sqrt{\lambda}$ which holds for the optimal solution. We show now that this is dynamically achieved to any desired accuracy after the elapse of sufficient time. In practice, however, it is in almost all cases achieved after one pass over the data. 

\begin{proposition} For $t>0$ the norm of the weight vector $\vec w_t$ is bounded from above as follows
\begin{equation}
\label{bound}
\left\|\vec{w}_t\right\|\le \frac{1}{\sqrt{\lambda}}\sqrt{1+\left(\frac{R^2}{\lambda}-1\right)\frac{1}{t}}\enspace.
\end{equation}
\end{proposition}
\begin{wrapfigure}{l}{0.55\textwidth}
\vspace{-0pt}
\epsfig{file=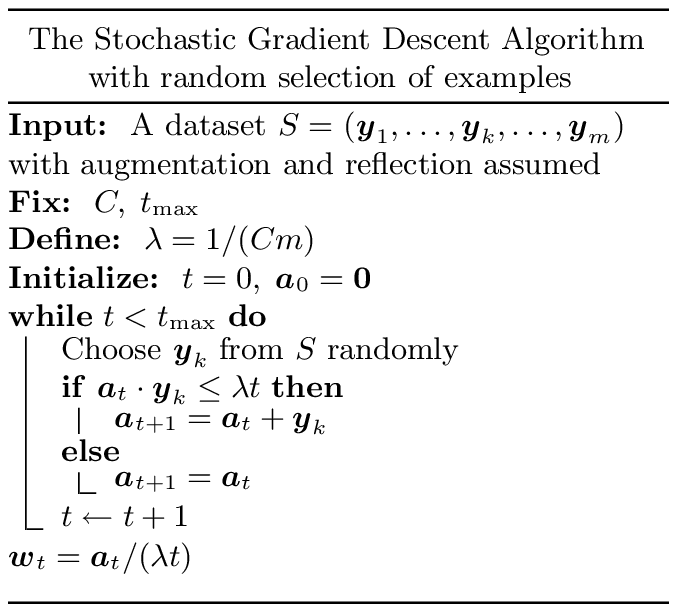, width=0.55\textwidth}
\vspace{-10pt}
\end{wrapfigure}
\begin{proof}
From the update rule (\ref{upd3}) taking into account condition (\ref{cond2}) under which the update takes place we get
\[
\left\|\vec{a}_{t+1}\right\|^2- \left\|\vec{a}_{t}\right\|^2 = \left \|\vec{y}_{k}\right \|^2
+2\vec{a}_t \cdot \vec{y}_{k} \le R^2+2\lambda t \enspace.
\]
Obviously, this is trivially satisfied if (\ref{cond2}) is violated and (\ref{upd4}) holds. A repeated application of the above inequality with $\vec a_0=\vec 0$ gives
\[
\left\|\vec{a}_{t}\right\|^2 \le R^2t+2\lambda \sum_{k=0}^{t-1}k=R^2t+\lambda t(t-1)=(R^2-\lambda)t+\lambda t^2
\]
from where using (\ref{var}) and taking the square root we obtain (\ref{bound}). \qed
\end{proof}

Combining (\ref{bound}) with the initial choice $\vec w_0=\vec0$ we see that for all $t$ the weaker bound $\left\|\vec{w}_t\right\| \le R/\lambda$ previously derived in \cite{SSS1} holds.

SGD gives naturally rise to online algorithms. Therefore, we may choose the examples to be presented to the algorithm at random. However, the L1-SVM optimization task is a batch learning problem which may be better tackled by online algorithms via the classical conversion of such algorithms to the batch setting. This is done by cycling repeatedly through the possibly randomly permuted training dataset and using the last hypothesis for prediction. This traditional procedure of presenting the training data to the algorithm in complete epochs has in our case, as we will see shortly, the additional advantage that there exists a lower bound on the optimal value of the objective function to be minimized which is expressed in terms of quantities available during the run. The existence of such a lower bound provides an estimate of the relative accuracy achieved by the algorithm. 

\begin{proposition}
\label{prop2}
Let us assume that at some stage the whole training set has been presented to the algorithm exactly $T$ times. Then, it holds that  
\begin{equation}
\label{lbound}
{\cal J}_{\rm opt}(C)\equiv \min_{\Tiny{\vec w}}{\cal J}(\vec w, C) \ge {\cal L}^T \equiv C \frac{M}{T}-\frac{1}{2} \left\|\vec{w}^T\right\|^2 \enspace,
\end{equation}
where $M$ is the total number of margin errors made up to that stage and $\vec{w}^T\equiv \vec w_{(mT)}$ the weight vector at $t=mT$ with $m$ being the size of the training set.
\end{proposition}

\begin{proof}
Let $I_k^t$ denote the number of margin errors made on the pattern  $\vec y_k$ up to time $t$ such that $\vec a_t=\sum_{k}I_k^t\vec y_k$. Obviously, it holds that
\begin{equation}
\label{dv1}
0 \le I_k^{(mT)} \le T
\end{equation}
since $\vec y_k$ up to time $t=mT$ has been presented to the algorithm exactly $T$ times. Then, taking into account (\ref{var}) we see that at time $t$ the dual variable $\alpha_k^t$ associated with $\vec y_k$ is $\alpha_k^t=I_k^t/(\lambda t)$ and consequently the dual variable $\alpha_k^{(mT)}$ after $T$ complete epochs is given by
\begin{equation}
\label{dv2}
\alpha_k^{(mT)}=\frac{I_k^{(mT)}}{\lambda mT}=C\frac{I_k^{(mT)}}{T} \enspace. 
\end{equation}
With use of (\ref{dv1}) we readily conclude that the dual variables after $T$ complete epochs automatically satisfy the box constraints
\begin{equation}
\label{box}
0 \le \alpha_k^{(mT)} \le C \enspace.
\end{equation}
From the weak duality theorem it follows that
\[
{\cal J}(\vec w, C)\ge{\cal L}(\vec \alpha)=\sum_k\alpha_k -\frac{1}{2}\sum_{i,j} \alpha_i \alpha_j \vec y_i \cdot \vec y_j \enspace,
\]
where ${\cal L}(\vec \alpha)$ is the dual Lagrangian\footnote{Maximization of ${\cal L}(\vec \alpha)$ subject to the constraints $0\le \alpha_k \le C$ is the dual of the primal L1-SVM problem expressed as a constrained minimization.} and the variables $\alpha_k$ obey the box constraints $0\le \alpha_k \le C$. Thus, setting $\alpha_k=\alpha^{(mT)}_k$ in the above inequality, noticing that $\sum_k\alpha_k^{(mT)}=$ $(C/T)\sum_k I_k^{(mT)}$ $=CM/T$ and substituting $\sum_{k}\alpha_k^{(mT)}\vec y_k$ with $\vec w^{T}$ we obtain ${\cal J}(\vec w, C)\ge{\cal L}^T$ which is equivalent to (\ref{lbound}). \qed
\end{proof}

In the course of proving Proposition \ref{prop2} we saw that although the algorithm is fully primal the dual variables $\alpha_k^t$ defined through the expansion $\vec w_t=\sum_{k}\alpha_k^t\vec y_k$ of the weight vector $\vec w_t$ as a linear combination of the patterns on which margin errors were made obey after $T$ complete epochs automatically the box constraints (\ref{box}) encountered in dual optimization.\footnote{We expect that the dual variables will also satisfy the box constraints in the limit $t \to \infty$ if the patterns presented to the algorithm are selected randomly with equal probability since asymptotically they will all be selected an equal number of times.} This surprising result allows us to construct the dual Lagrangian ${\cal L}^T$ which provides a lower bound on the optimal value ${\cal J}_{\rm opt}$ of the objective $\cal J$ and assuming ${\cal L}^T>0$ to obtain an upper bound ${\cal J}/{\cal L}^T-1$ on the relative accuracy ${\cal J}/{\cal J}_{\rm opt}-1$ achieved as the algorithm keeps running. Thus, we have for the first time a primal SGD algorithm which may use the relative accuracy as stopping criterion.\footnote{It is, of course, computationally expensive to evaluate at the end of each epoch the exact primal objective. Thus, an approximate calculation of the loss using the value that the weight vector had the last time each pattern was presented to the algorithm is preferable. This way we exploit the already computed inner product $\vec a_t \cdot \vec y_k$ which is needed in order to decide whether condition (\ref{cond2}) is satisfied. If this approximate calculation gives a value of the relative accuracy which is not larger than $f$ times the one set as stopping criterion we proceed to a proper calculation of the primal objective. The comparison coefficient $f$ is given empirically a value close to 1.} It is also worth noticing that ${\cal L}^T$ involves only the total number $M$ of margin errors and does not require that we keep the values of the individual dual variables during the run.

Although the automatic satisfaction of the box constraints by the dual variables is very important it is by no means sufficient to ensure vanishing of the duality gap and consequently convergence to the optimal solution. To demonstrate convergence to the optimum relying on dual optimization theory we must make sure that the Karush-Kuhn-Tucker (KKT) conditions \cite{Vap,CST} are satisfied. Their approximate satisfaction demands that the only patterns which have a substantial loss be the ones which have dual variables equal or at least extremely close to $C$ (bound support vectors) and moreover that the patterns which have zero loss and margin considerably larger than $1/\left\|\vec{w}^T\right\|$ should have vanishingly small dual variables. Patterns with margin very close to $1/\left\|\vec{w}^T\right\|$ may \begin{wrapfigure}{l}{0.55\textwidth}
\vspace{-0pt}
\epsfig{file=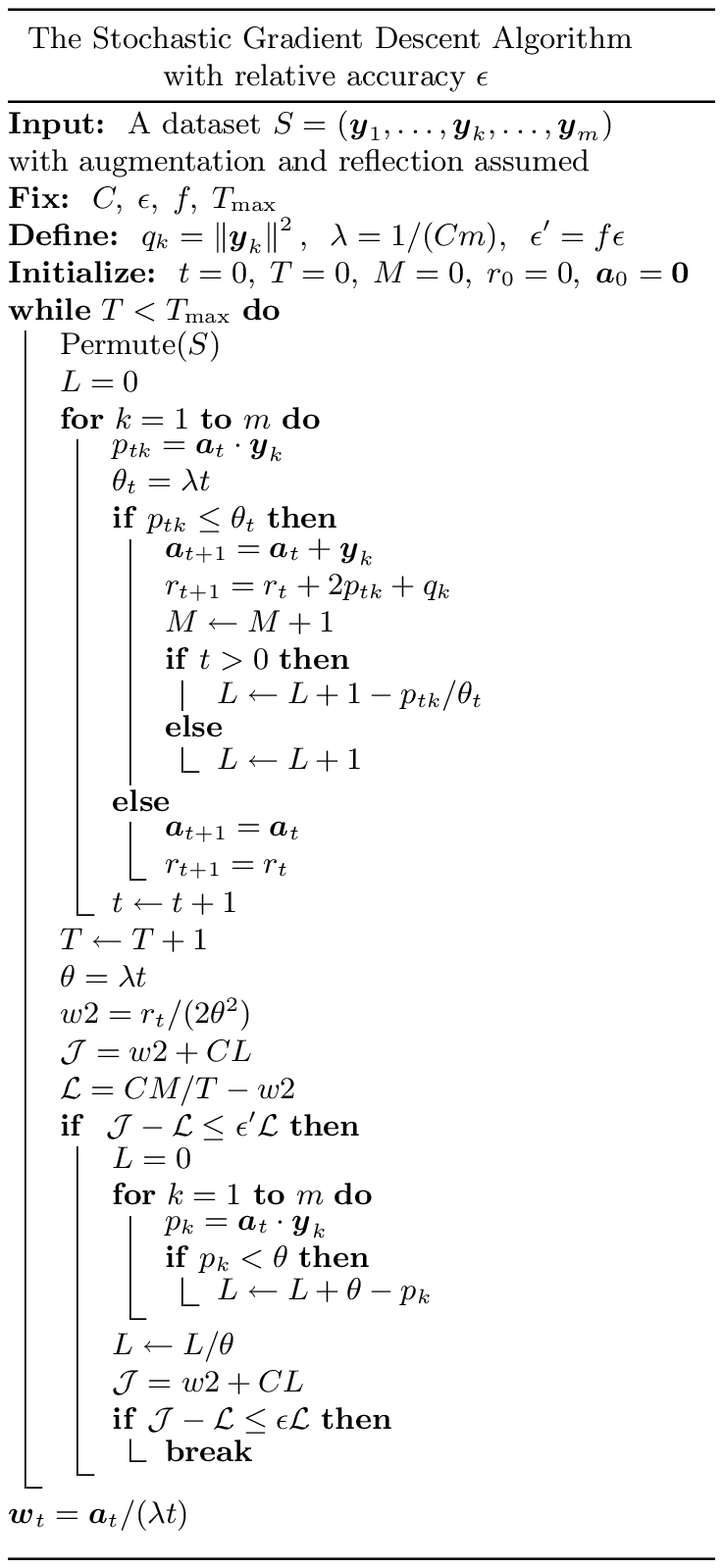, width=0.55\textwidth}
\vspace{-10pt}
\end{wrapfigure} have dual variables with values between 0 and $C$ and play the role of the non-bound support vectors. From (\ref{dv2}) we see that the dual variable associated with the $k$-th pattern is equal to $CT_k/T$ where $T_k \equiv I_k^{(mT)}$ is the number of epochs for which the $k$-th pattern was found to be a margin error. 
It is apparent that if there exists a number of epochs no matter how large it may be after which a pattern is consistently found to be a margin error then in the limit $T \to \infty$ we will have $(T_k/T) \to 1$ and the dual variable associated with it will asymptotically approach $C$. In contrast, if a pattern after a specific number of epochs is never found to be a margin error then $(T_k/T) \to 0$ and its dual variable will tend asymptotically to zero reflecting the accumulated effect of the shrinking that the weight vector suffers each time a pattern is presented to the algorithm. Therefore, the algorithm has the necessary ingredients for asymptotic satisfaction of the KKT conditions for the vanishing of the duality gap. The potential danger remains, however, that they may exist patterns with margin not very close to $1/\left\|\vec{w}^T\right\|$ which do not belong to any of the above categories and occasionally either become margin errors although most of the time are not or become classified with sufficiently large margin despite of the fact that they are most of the time margin errors. The hope is that with time the changes in the weight vector $\vec w_t$ will become smaller and smaller and such events will become more and more rare leading eventually to convergence to the optimal solution.

The above discussion cannot be regarded as a formal proof of the asymptotic convergence of the algorithm. We believe, however, that it does provide a convincing argument that assuming convergence (not necessarily to the optimum) the duality gap will eventually tend to zero and the lower bound ${\cal L}^T$ on the primal objective ${\cal J}$ given in Proposition \ref{prop2} will approach the optimal primal objective ${\cal J}_{\rm opt}$, thereby proving that convergence to the optimum has been achieved. If, instead, we make the stronger assumption of convergence to the optimum then, of course, the vanishing of the duality gap follows from the strong duality theorem. In any case the stopping criterion exploiting the upper bound ${\cal J}/{\cal L}^T-1$ on the relative accuracy ${\cal J}/{\cal J}_{\rm opt}-1$ is a meaningful one.

Our discussion so far assumes that in an epoch each pattern is presented only once to the algorithm. We may, however, consider the option of presenting the same pattern $\vec y_k$ repeatedly $\ell$ times to the algorithm\footnote{Multiple updates were introduced in \cite{PT2,PT1}. A discussion in a context related to the present work is given in \cite{KL}. However, a proper SGD treatment in the presence of a regularization term for the 1-norm soft margin loss was not provided. Instead, a ``forward-backward splitting'' approach was adopted in which a multiple update in the absence of the regularizer is followed by $\ell$ pure regularizer-induced $\vec w_t$ shrinkings.} aiming at performing more updates at the expense of only one calculation of the costly inner product $\vec a_t \cdot \vec y_k$. Proposition \ref{prop2} and the analysis following it will still be valid on the condition that all patterns in each epoch are presented exactly the same number $\ell$ of times to the algorithm. Then, such an epoch should be regarded as equivalent to $\ell$ usual epochs with single presentations of patterns to the algorithm and will have as a result the increase of $t$ by an amount equal to $m\ell$.

It is, of course, important to be able to decide in terms of just the initial value of $\vec a_t \cdot \vec y_k$ how many, let us say $\ell_+$, out of these $\ell$ consecutive presentations of the pattern $\vec y_k$ to the algorithm will lead to a margin error, i.e., to an update of $\vec a_t$, with each of the remaining $\ell_-=\ell-\ell_+$ presentations necessarily corresponding to just an increase of $t$ by 1 which amounts to a pure shrinking of $\vec w_t$.

\begin{proposition}
\label{propm}
Let the pattern $\vec y_k$ be presented at time $t$ repeatedly $\ell$ times to the algorithm. Also  let
\[
P=\vec a_t \cdot \vec y_k-\lambda t \enspace.
\]
Then, the number $\ell_+$ of times that $\vec y_k$ will be found to be a margin error is given by the following formula
\begin{eqnarray}
\label{mult}
&{\rm {if}} \; &P>(\ell-1)\lambda \ \ \ \ \ \  \ell_+=0 \nonumber \enspace,\\
&{\rm {if}} \; &P\le(\ell-1)\lambda \ \ \ \ \ \ \ell_+=\min \left \{\ell,\left [\frac{(\ell-1)\lambda-P}{\max \{\left \|\vec{y}_{k}\right \|^2, \lambda \}}\right ]+1\right \}  \enspace.
\end{eqnarray}
Here $[x]$ denotes the integer part of $x \ge 0$.
\end{proposition}

\begin{proof}
For the sake of brevity we call a plus-step a presentation of the pattern $\vec y_k$ to the algorithm which leads to a margin error and a minus-step a presentation which does not. If at time $t$ a plus-step takes place $\vec a_{t+1} \cdot \vec y_k-\lambda (t+1)=(\vec a_t \cdot \vec y_k-\lambda t)+(\left \|\vec{y}_{k}\right \|^2-\lambda)$ while if a minus-step takes place $\vec a_{t+1} \cdot \vec y_k-\lambda (t+1)=(\vec a_t \cdot \vec y_k-\lambda t)-\lambda$. Thus, a plus-step adds to $P$ the quantity $\left \|\vec{y}_{k}\right \|^2-\lambda$ while a minus-step the quantity $-\lambda$. Clearly, after $\ell$ consecutive presentations of $\vec y_k$ to the algorithm it holds that $\vec a_{t+\ell} \cdot \vec y_k-\lambda (t+\ell)=P+\ell_+(\left \|\vec{y}_{k}\right \|^2-\lambda)-(\ell-\ell_+)\lambda$.

If $P>(\ell-1)\lambda$ it follows that $P-(\ell-1)\lambda>0$ which means that after $\ell-1$ consecutive minus-steps condition (\ref{cond2}) is still violated and an additional minus-step must take place. Thus, $\ell_-=\ell$ and $\ell_+=0$.

For $P \le (\ell-1)\lambda$ we first treat the subcase $\max \{\left \|\vec{y}_{k}\right \|^2, \lambda \}=\lambda$. If $\left \|\vec{y}_{k}\right \|^2 \le \lambda$ and $P \le 0$ condition (\ref{cond2}) is initially satisfied and will still be satisfied after any number of plus-steps since the quantity $\left \|\vec{y}_{k}\right \|^2-\lambda$ that is added to $P$ with a plus-step is non-positive. Thus, $\ell_+=\ell$. This is in accordance with (\ref{mult}) since $((\ell-1)\lambda-P)/\lambda \ge \ell-1$ or $[((\ell-1)\lambda-P)/\lambda]+1\ge \ell$ leading to $\ell_+=\ell$. It remains for $\left \|\vec{y}_{k}\right \|^2 \le \lambda$ to consider $P$ in the interval $0<P\le (\ell-1)\lambda$ which can be further subdivided as $(\ell_1-1)\lambda < P\le \ell_1\lambda$ with the integer $\ell_1$ satisfying $1\le \ell_1 \le \ell-1$. For $P$ belonging to such a subinterval condition (\ref{cond2}) is initially violated and will still be violated after $\ell_1-1$ minus-steps while after one more minus-step will be satisfied. It will still be satisfied after any number of additional plus-steps because the quantity $\left \|\vec{y}_{k}\right \|^2-\lambda$ that is added to $P$ with a plus-step is non-positive. Thus, $\ell_-=\ell_1$ and $\ell_+=\ell-\ell_1$. This is in accordance with (\ref{mult}) since $(\ell-\ell_1-1)\lambda \le(\ell-1)\lambda -P <(\ell-\ell_1)\lambda$ leads to $[((\ell-1)\lambda -P)/\lambda]+1=\ell-\ell_1$.

The subcase $\left \|\vec{y}_{k}\right \|^2>\lambda$ of the case $P \le (\ell-1)\lambda$ is far more complicated. If $\left \|\vec{y}_{k}\right \|^2>\lambda$ with $P\le -(\ell-1)(\left \|\vec{y}_{k}\right \|^2-\lambda)$ condition (\ref{cond2}) is initially satisfied and will still be satisfied after $\ell-1$ plus-steps since $P+(\ell-1)(\left \|\vec{y}_{k}\right \|^2-\lambda)\le 0$. Thus, $\ell_+=\ell$. This is consistent with (\ref{mult}) because $(\ell-1)\lambda-P\ge (\ell-1)\left \|\vec{y}_{k}\right \|^2$ or $[((\ell-1)\lambda-P)/\left \|\vec{y}_{k}\right \|^2]+1 \ge \ell$ leading to $\ell_+=\ell$. It remains to be examined the case $\left \|\vec{y}_{k}\right \|^2>\lambda$ with $P$ in the interval $-(\ell-1)(\left \|\vec{y}_{k}\right \|^2-\lambda)<P \le (\ell-1)\lambda$. The above interval can be expressed as a union of subintervals $(\ell-\ell_1-1)\lambda-\ell_1(\left \|\vec{y}_{k}\right \|^2-\lambda) <P \le (\ell-\ell_1)\lambda-(\ell_1-1)(\left \|\vec{y}_{k}\right \|^2-\lambda)$ with the integer $\ell_1$ satisfying $1\le \ell_1 \le \ell-1$. Let $P$ belong to such a subinterval. Let us also assume that the pattern $\vec{y}_{k}$ has been presented $\kappa \le \ell$ consecutive times to the algorithm as a result of which $\kappa_+$ plus-steps and $\kappa_-$ minus-steps have taken place and the quantity $\kappa_+(\left \|\vec{y}_{k}\right \|^2-\lambda)-\kappa_-\lambda$ has been added to $P$. Then $P_{\kappa_+,\kappa_-} \equiv P+\kappa_+(\left \|\vec{y}_{k}\right \|^2-\lambda)-\kappa_-\lambda$ satisfies $(\ell-\ell_1-1-\kappa_-)\lambda-(\ell_1-\kappa_+)(\left \|\vec{y}_{k}\right \|^2-\lambda) <P_{\kappa_+,\kappa_-} \le (\ell-\ell_1-\kappa_-)\lambda-(\ell_1-1-\kappa_+)(\left \|\vec{y}_{k}\right \|^2-\lambda)$. As $\kappa$ increases either $\kappa_+$ will first reach the value $\ell_1$ with $\kappa_-<\ell-\ell_1$ or $\kappa_-$ will first reach the value $\ell-\ell_1$ with $\kappa_+<\ell_1$. In the former case $0\le(\ell-\ell_1-1-\kappa_-)\lambda<P_{\kappa_+,\kappa_-}$. This means that condition (\ref{cond2}) is violated and will continue being violated until the number of minus-steps becomes equal to $\ell-\ell_1-1$ in which case one more minus-step must take place. Thus, all steps taking place after $\kappa_+$ has reached the value $\ell_1$ are minus-steps. In the latter case $P_{\kappa_+,\kappa_-} \le -(\ell_1-1-\kappa_+)(\left \|\vec{y}_{k}\right \|^2-\lambda)\le 0$. This means that condition (\ref{cond2}) is satisfied and will continue being satisfied until the number of plus-steps becomes equal to $\ell_1-1$ in which case one more plus-step must take place. Thus, all steps taking place after $\kappa_-$ has reached the value $\ell-\ell_1$ are plus-steps. In both cases $\ell_+=\ell_1$. This is again in accordance with (\ref{mult}) because $(\ell_1-1)\left \|\vec{y}_{k}\right \|^2 \le (\ell-1)\lambda-P < \ell_1\left \|\vec{y}_{k}\right \|^2$ or $[((\ell-1)\lambda-P)/\left \|\vec{y}_{k}\right \|^2]+1 = \ell_1$. \qed
\end{proof}

With $\ell_+$ given in Proposition \ref{propm} the update of multiplicity $\ell$ of the weight vector $\vec a_t$ is written formally as 
\begin{equation}
\label{updm}
\vec a_{t+\ell}= \vec a_t+ \ell_+ \vec y_k \enspace.
\end{equation}

\section{Implementation and Experiments}
We implement three types of SGD algorithms\footnote{Sources available at \url{http://users.auth.gr/costapan}} along the lines of the previous section. The first is the plain algorithm with random selection of examples, denoted SGD-r, which terminates when the maximum number $t_{\rm max}$ of steps is reached. Its pseudocode is given in Section 2. The dual variables in this case do not satisfy the box constraints as a result of which relative accuracy cannot be used as stopping criterion. The SGD algorithm with relative accuracy $\epsilon$, the pseudocode of which is also given in Section 2, is denoted SGD-s where s designates that in an epoch each pattern is presented a single time to the algorithm. It terminates when the relative deviation of the primal objective $\cal J$ from the dual Lagrangian ${\cal L}^T$ just falls below $\epsilon$ provided the maximum number $T_{\rm max}$ of full epochs is not exhausted. A variation of this algorithm, denoted SGD-m, replaces in the $T$-th epoch the usual update with the multiple update (\ref{updm}) of multiplicity $\ell=5$ only if $0<T \;{\rm mod}\; 9<5$. For both SGD-s and SGD-m the comparison coefficient takes the value $f=1.2$ unless otherwise explicitly stated.

Algorithms performing SGD on the primal objective are expected to perform better if linear kernels are employed. Therefore the feature space in our experiments will be chosen to be the original instance space. As a consequence, our algorithms should most naturally be compared with linear SVMs. Among them we choose ${\rm SVM}^{\rm perf}$ \footnote{Source (version 2.50) available at \url{http://svmlight.joachims.org}} \cite{Joa06}, the first cutting-plane algorithm for training linear SVMs, the Optimized Cutting Plane Algorithm for SVMs\footnote{Source (version 0.96) available at \url{http://cmp.felk.cvut.cz/~xfrancv/ocas/html}} (OCAS) \cite{FS}, the Dual Coordinate Descent\footnote{Source available at \url{http://www.csie.ntu.edu.tw/~cjlin/liblinear}. We used the slightly faster older liblinear version 1.7 instead of the latest 1.93.} (DCD) algorithm \cite{HCL} and the Margin Perceptron with Unlearning\footnote{Source available at \url{http://users.auth.gr/costapan}} (MPU) \cite{PT2}. We also include in our study Pegasos\footnote{Source available at \url{http://ttic.uchicago.edu/~shai/code}} ($k=1$). Finally, we briefly considered the SvmSgd\footnote{Source (version 2) available at \url{http://leon.bottou.org/projects/sgd}} \cite{Bot} and SGD-QN\footnote{Source available at \url{https://www.hds.utc.fr/~bordesan/dokuwiki/doku.php?id=en:sgdqn}} \cite{Bor} algorithms implemented in single precision.

\begin{table}[t]
\centering
\caption{The number $T$ of complete epochs required in order for the SGD-s algorithm to achieve $(\cal J-{\cal L}^T)/{\cal L}^T$ $\le 10^{-5}$ for $C=0.1$.}
\label{Table1}
\includegraphics[width=0.83\textwidth]{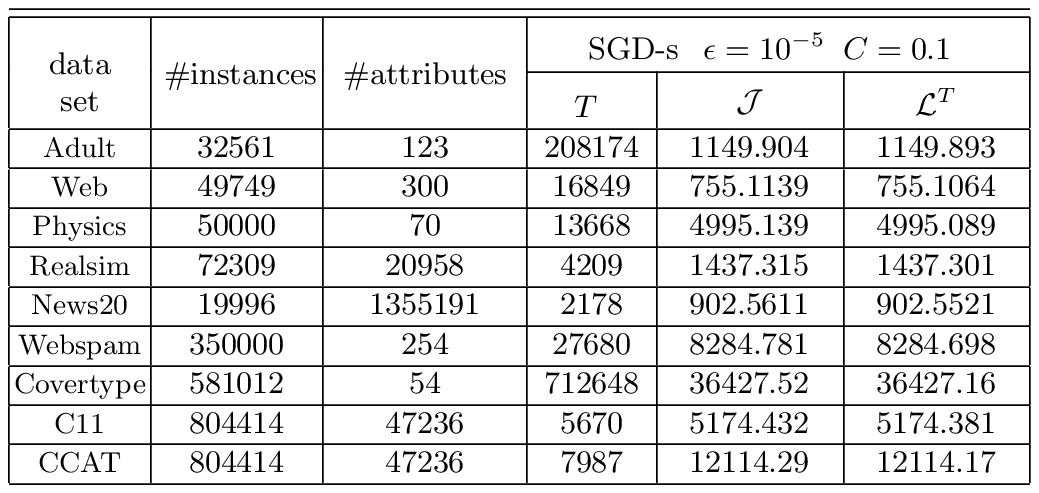}
\end{table}
The datasets we used for training are the binary Adult and Web datasets as compiled by Platt,\footnote{\url{http://research.microsoft.com/en-us/projects/svm/}} the training  set of the KDD04 Physics dataset\footnote{\url{http://osmot.cs.cornell.edu/kddcup/datasets.html}} (with 70 attributes after removing the 8 columns containing missing features), the Real-sim, News20 and Webspam (unigram treatment) datasets,\footnote{\url{http://www.csie.ntu.edu.tw/~cjlin/libsvmtools/datasets}} the multiclass Covertype UCI dataset\footnote{\url{http://archive.ics.uci.edu/ml/datasets.html}} and the full Reuters RCV1 dataset.\footnote{\url{http://www.jmlr.org/papers/volume5/lewis04a/lyrl2004_rcv1v2_README.htm}} Their number of instances and attributes are listed in Table \ref{Table1}. In the case of the Covertype dataset we study the binary classification problem of the first class versus rest while for the RCV1 we consider both the binary text classification tasks of the C11 and CCAT classes versus rest. The Physics and Covertype datasets were rescaled by multiplying all the features with 0.001. The experiments were conducted on a 2.5 GHz Intel Core 2 Duo processor with 3 GB RAM running Windows Vista. The C++ codes were compiled using the g++ compiler under Cygwin.

First we perform an experiment aiming at demonstrating that our SGD algorithms are able to obtain extremely accurate solutions. More specifically, with the algorithm SGD-s employing single updating we attempt to diminish the gap between the primal objective $\cal J$ and the dual Lagrangian ${\cal L}^T$ setting as a goal a relative deviation $(\cal J-{\cal L}^T)/{\cal L}^T$ $\le 10^{-5}$ for $C=0.1$. In the present and in all subsequent experiments we do not include a bias term in any of the algorithms (i.e., in our case we assign to the augmentation parameter the value $\rho =0$). In order to keep the number $T$ of complete epochs as low as possible we increase the comparison coefficient $f$ until the number of epochs required gets stabilized. This procedure does not entail, of course, the shortest training time but this is not our concern in this experiment. In Table \ref{Table1} we give the values of both $\cal J$ and ${\cal L}^T$ and the number $T$ of epochs needed to achieve these values. If multiple updates are used a larger number of epochs is, in general, required due to the slower increase of ${\cal L}^T$. Thus, SGD-s achieves, in general, relative accuracy closer to $\epsilon$ than SGD-m does. This is confirmed by subsequent experiments.

\begin{table}[t]
\centering
\caption{Training times of SGD algorithms to achieve $({\cal J}-{\cal J}_{\rm opt})/{\cal J}_{\rm opt} \le 0.01$ for $C=1$.}
\label{Table2}
\includegraphics[width=0.92\textwidth]{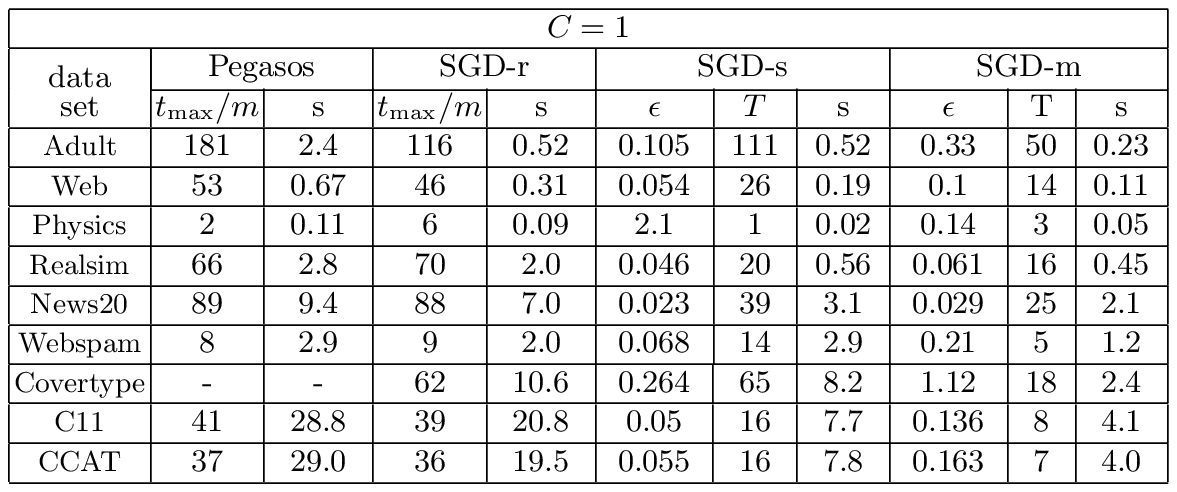}
\end{table}
\begin{table}[t]
\centering
\caption{Training times of linear SVMs to achieve $({\cal J}-{\cal J}_{\rm opt})/{\cal J}_{\rm opt} \le 0.01$ for $C=1$.}
\label{Table3}
\includegraphics[width=0.70\textwidth]{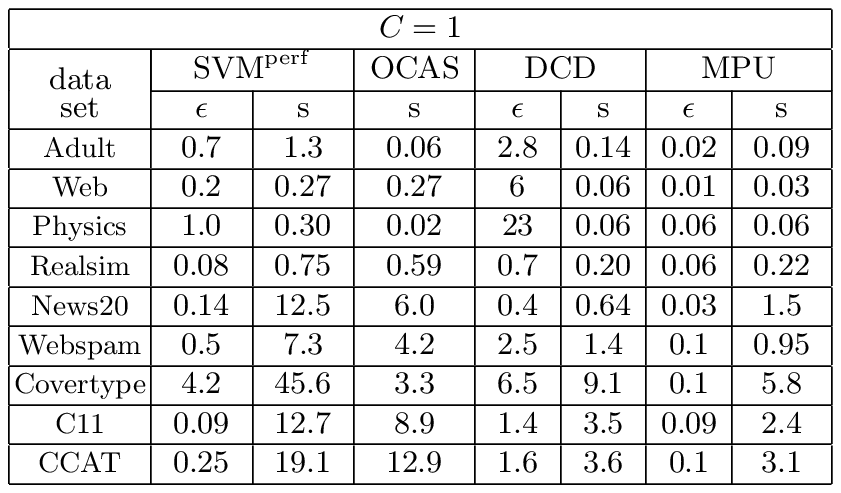}
\vspace{-10pt}
\end{table}
In our comparative experimental investigations we aim at achieving relative accuracy $({\cal J}-{\cal J}_{\rm opt})/{\cal J}_{\rm opt} \le 0.01$ for various values of the penalty parameter $C$ assuming knowledge of the value of ${\cal J}_{\rm opt}$. For Pegasos and SGD-r we use as stopping criterion the exhaustion of the maximum number of steps (iterations) $t_{\rm max}$ which, however, is given values which are multiples of the dataset size $m$. The ratio $t_{\rm max}/m$ may be considered analogous to the number $T$ of epochs of the algorithm SGD-s since equal values of these two quantities indicate identical numbers of ${\vec w}_t$ updates. The input parameter for SGD-s and SGD-m is the (upper bound on) the relative accuracy $\epsilon$. For MPU we use the parameter $\epsilon=\delta=\delta_{\rm stop}$, where $\delta$ is the before-run relative accuracy  and $\delta_{\rm stop}$ the stopping threshold for the after-run relative accuracy. For ${\rm SVM}^{\rm perf}$ and DCD we use as input their parameter $\epsilon$ while for OCAS the primal objective value $q=1.01{\cal J}_{\rm opt}$ (not given in the tables) with the relative tolerance taking the default value $r=0.01$. Any difference in training time between Pegasos and SGD-r for equal values of $t_{\rm max}/m$ should be attributed to the difference in the implementations. Any difference between $t_{\rm max}/m$ for SGD-r and $T$ for SGD-s is to be attributed to the different procedure of choosing the patterns that are presented to the algorithm. Finally, the difference in the number $T$ of epochs between SGD-s and SGD-m reflects the effect of multiple updates. It should be noted that in the runtime of SGD-s and SGD-m several calculations of the primal and the dual objective are included which are required for checking the satisfaction of the stopping criterion. If SGD-s and SGD-m were using the exhaustion of the maximum number $T_{\rm max}$ of epochs as stopping criterion their runtimes would certainly be shorter.

\begin{table}[t]
\centering
\caption{Training times of SGD algorithms to achieve $({\cal J}-{\cal J}_{\rm opt})/{\cal J}_{\rm opt} \le 0.01$ for $C=10$.}
\label{Table4}
\includegraphics[width=0.92\textwidth]{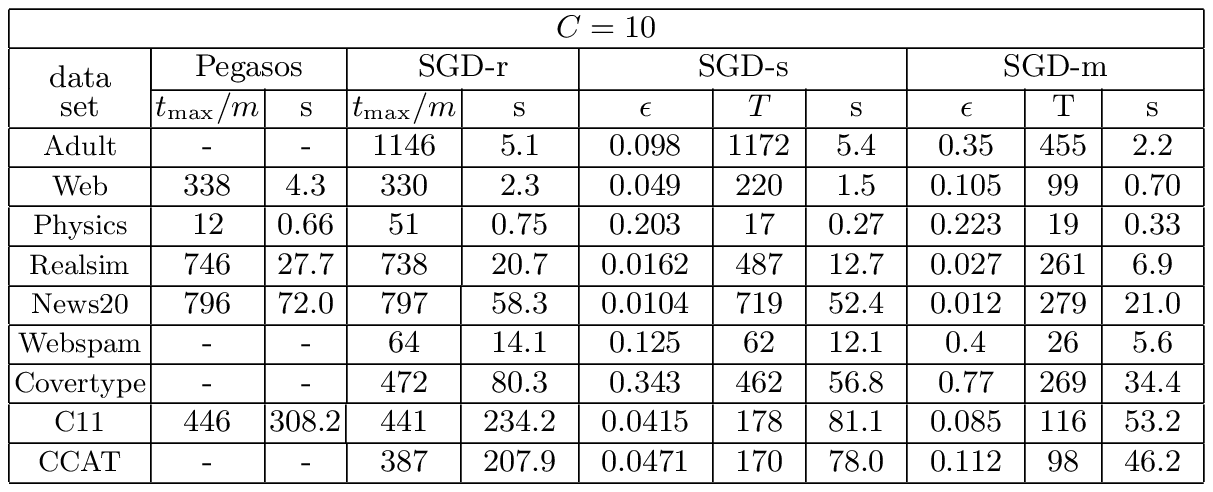}
\end{table}
\begin{table}[t]
\centering
\caption{Training times of linear SVMs to achieve $({\cal J}-{\cal J}_{\rm opt})/{\cal J}_{\rm opt} \le 0.01$ for $C=10$.}
\label{Table5}
\includegraphics[width=0.70\textwidth]{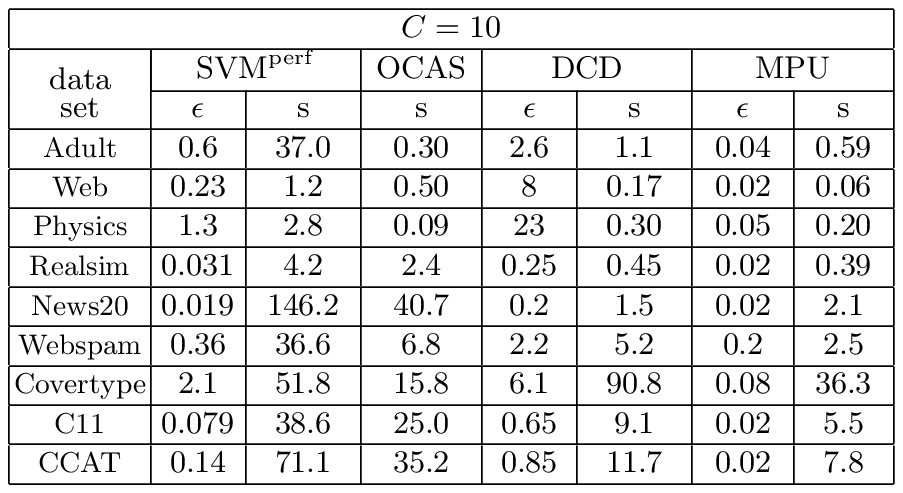}
\vspace{-10pt}
\end{table}
Tables \ref{Table2} and \ref{Table3} contain the results of the experiments involving the SGD algorithms and linear SVMs for $C=1$. We observe that, in general, there is a progressive decrease in training time as we move from Pegasos to SGD-m through SGD-r and SGD-s due to the additive effect of several factors. These factors are the more efficient implementation of our algorithms exploiting the change of variable given by (\ref{var}), the presentation of the patterns to SGD-s and SGD-m in complete epochs (see also \cite{Bot,SSS1}) and the use by SGD-m of multiple updating. The overall improvement made by SGD-m over Pegasos is quite substantial. DCD and MPU are certainly statistically faster but their differences from SGD-m are not very large especially for the largest datasets. Moreover, SGD-s and SGD-m are considerably faster than ${\rm SVM}^{\rm perf}$ and statistically faster than OCAS. Pegasos failed to process the Covertype dataset due to numerical problems. 

\begin{table}[t]
\centering
\caption{Training times of SGD algorithms to achieve $({\cal J}-{\cal J}_{\rm opt})/{\cal J}_{\rm opt} \le 0.01$ for  $C=0.05$.}
\label{Table6}
\includegraphics[width=0.92\textwidth]{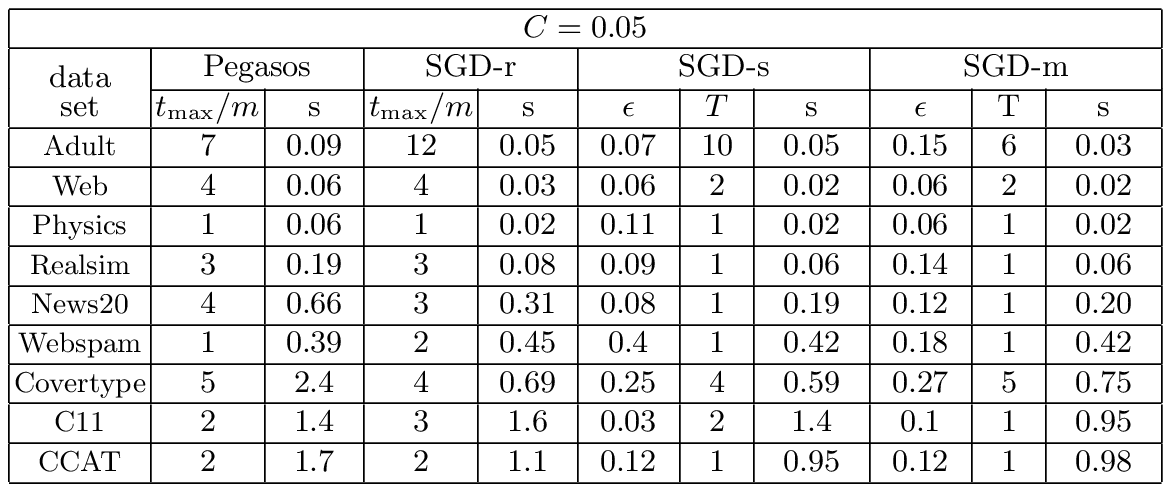}
\end{table}
\begin{table}[t]
\centering
\caption{Training times of linear SVMs to achieve $({\cal J}-{\cal J}_{\rm opt})/{\cal J}_{\rm opt} \le 0.01$ for $C=0.05$.}
\label{Table7}
\includegraphics[width=0.70\textwidth]{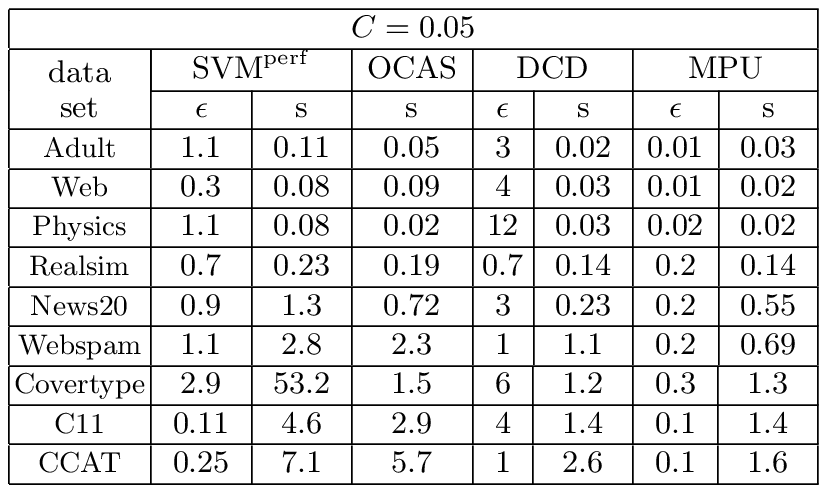}
\vspace{-10pt}
\end{table}
Tables \ref{Table4} and \ref{Table5} contain the results of the experiments involving the SGD algorithms and linear SVMs for $C=10$. Although the general characteristics resemble the ones of the previous case the differences are magnified due to the intensity of the optimization task. Certainly, the training time of linear SVMs scales much better as $C$ increases. Moreover, MPU clearly outperforms DCD and OCAS for most datasets. SGD-m is still statistically faster than ${\rm SVM}^{\rm perf}$ but slower than OCAS. Finally, Pegasos runs more often into numerical problems. 

In contrast, as $C$ decreases the differences among the algorithms are alleviated. This is apparent from the results for $C=0.05$ reported in Tables \ref{Table6} and \ref{Table7}. SGD-r, SGD-s and SGD-m all appear statistically faster than the linear SVMs. Also Pegasos outperforms ${\rm SVM}^{\rm perf}$ for all datasets and OCAS for the majority of them. Seemingly, lowering $C$ favors the SGD algorithms.

\begin{table}[t]
\centering
\caption{Training times of the algorithms SvmSgd, SGD-QN, SGD-m and MPU implemented in single precision to achieve $({\cal J}-{\cal J}_{\rm opt})/{\cal J}_{\rm opt} \le 0.01$ for $C=1$.}
\label{Table8}
\includegraphics[width=0.85\textwidth]{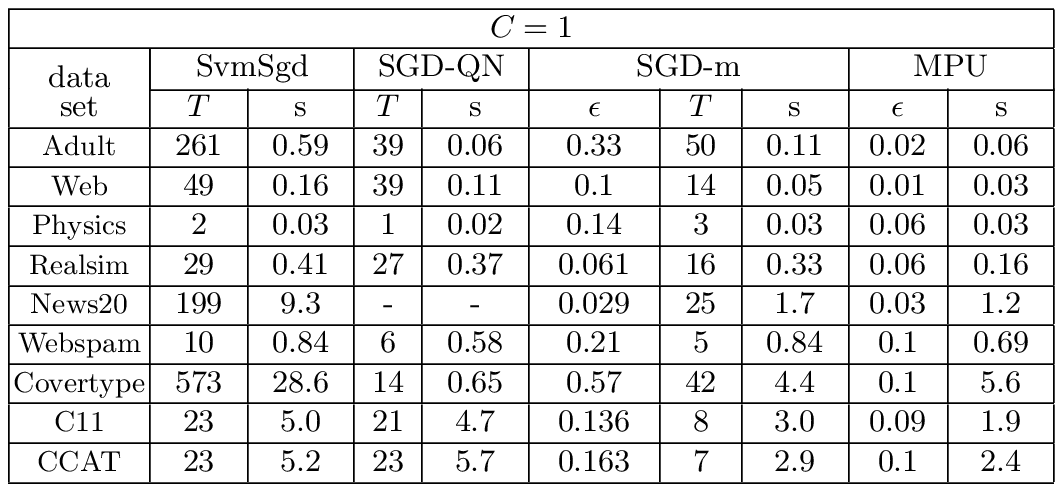}
\vspace{-10pt}
\end{table}
Before concluding our experimental investigation we also consider the SGD algorithms SvmSgd and SGD-QN both implemented in single precision. For a fair comparison we implemented the algorithms SGD-m and MPU in single precision as well. SvmSgd and SGD-QN do not perform random permutations of the dataset but rather assume that it has already been shuffled. Thus, we provided them with the dataset produced by SGD-m as a result of the first random permutation of the dataset given. The fact that the computational cost of the random permutation is not included in the runtime of SvmSgd and SGD-QN gives these algorithms a certain advantage which becomes more crucial for tasks requiring short runtimes. For this reason we did not consider values of the penalty parameter $C$ much smaller than 1. Table \ref{Table8} contains the results of the experiments involving the algorithms SvmSgd, SGD-QN, SGD-m and MPU for $C=1$. We observe that statistically SvmSgd is the slowest and MPU the fastest. Moreover, SGD-m statistically outperforms SGD-QN with a preference for sparse multidimentional datasets.  In the case of the News20 dataset SGD-QN failed to reach the required relative accuracy.

\section{Conclusions}
We reexamined the classical SGD approach to the primal unconstrained L1-SVM optimization task and made some contributions concerning both theoretical and practical issues. Assuming a learning rate inversely proportional to the number of steps a simple change of variable allowed us to simplify the algorithmic description and demonstrate that in a scheme presenting the patterns to the algorithm in complete epochs the naturally defined dual variables satisfy automatically the box constraints of the dual optimization. This opened the way to obtaining an estimate of the progress made in the optimization process and enabled the adoption of a meaningful stopping criterion, something the SGD algorithms were lacking. Moreover, it made possible a qualitative discussion of how the KKT conditions will be asymptotically satisfied provided the weight vector $\vec{w}_t$ gets stabilized. Besides, we showed that in the limit ${t \to \infty}$ even without a projection step in the update it holds that $\left\|\vec{w}_t\right\| \le 1/\sqrt{\lambda}$, a bound known to be obeyed by the optimal solution. On the more practical side by exploiting our simplified algorithmic description and employing a mechanism of multiple updating we succeeded in substantially improving the performance of SGD algorithms. For optimization tasks of low or medium intensity the algorithms constructed are comparable to or even faster than the state-of-the-art linear SVMs.


\newpage 
\begin{thebibliography}{}

\bibitem{Bos}Boser, B., Guyon, I., Vapnik, V.: A training algorithm for optimal margin classifiers. In: 
COLT, pp. 144--152 (1992)

\bibitem{Bor}Bordes, A., Bottou, L., Gallinari, P.: SGD-QN: Careful quasi-Newton stochastic gradient descent. JMLR {\bf 10}, 1737--1754 (2009)

\bibitem{Bot}Bottou, L. (Web Page). Stochastic gradient descent examples. \url{http://leon.bottou.org/projects/sgd}

\bibitem{Cortes}Cortes, C., Vapnik, V.: Support vector networks. Mach Learn {\bf 20}, 273--297 (1995) 

\bibitem{CST}Cristianini, N., Shawe-Taylor, J.: An Introduction to Support Vector Machines. Cambridge University Press, Cambridge (2000)

\bibitem{DH}Duda, R.O., Hart, P.E.: Pattern Classsification and Scene Analysis. Wiley, Chichester (1973) 

\bibitem{FS}Frank, V., Sonnenburg, S.: Optimized cutting plane algorithm for support vector machines. In: ICML, pp. 320-327 (2008) 

\bibitem{HCL}Hsieh, C.-J., Chang, K.-W., Lin, C.-J., Keerthi, S.S., Sundararajan, S.: A dual coordinate descent method for large-scale linear SVM. In: ICML, pp. 408--415 (2008) 

\bibitem{Joa}
Joachims, T.: Making large-scale SVM learning practical. In: Advances in Kernel Methods-Support Vector Learning. MIT Press, Cambridge (1999)  

\bibitem{Joa06}
Joachims, T.: Training linear SVMs in linear time. In: KDD, pp. 217--226 (2006) 

\bibitem{KL}
Karampatziakis, N., Langford, J.: Online importance weight aware updates. In: UAI, pp. 392--399 (2011) 

\bibitem{KSW}
Kivinen, J., Smola, A., Williamson, R.: Online learning with kernels. IEEE Transactions on Signal Processing {\bf 52(8)}, 2165-2176 (2004)

\bibitem{PT2}
Panagiotakopoulos, C., Tsampouka, P.: The margin perceptron with unlearning. In: ICML, pp. 855-862 (2010) 

\bibitem{PT1}
Panagiotakopoulos, C., Tsampouka, P.: The margitron: A generalized perceptron with margin. IEEE Transactions on Neural Networks {\bf 22(3)}, 395-407 (2011)

\bibitem{PT3}
Panagiotakopoulos, C., Tsampouka, P.: The perceptron with dynamic margin. In: Kivinen, J., et. al. (eds.) ALT 2011. LNCS (LNAI) vol. 6925, pp. 204-218. Springer, Heidelberg (2011)

\bibitem{Plat}
Platt, J.C.: Sequential minimal optimization: A fast algorithm for training support vector machines. Microsoft Res. Redmond WA, Tech. Rep. MSR-TR-98-14 (1998)

\bibitem{Ros}
Rosenblatt, F.: The perceptron: A probabilistic model for information storage and
organization in the brain. Psychological Review, {\bf 65 (6)}, 386--408 (1958) 
  
\bibitem{SSS}
Shalev-Schwartz, S., Singer, Y., Srebro, N.: Pegasos: Primal estimated sub-gradient solver for SVM. In: ICML, pp.  807--814 (2007)  

\bibitem{SSS1}
Shalev-Schwartz, S., Singer, Y., Srebro, N., Cotter, A.: Pegasos: Primal estimated sub-gradient solver for SVM. Mathematical Programming, {\bf 127(1)}, 3-30 (2011) 

\bibitem{Vap}
Vapnik, V.: Statistical learning theory. Wiley, Chichester (1998)

\bibitem{Zh}
Zhang, T.: Solving large scale linear prediction problems using stochastic gradient descent algorithms. In: ICML (2004) 

\end{thebibliography}
\end{document}